%% file: samplepaper.tex
\documentclass[runningheads]{llncs}
\usepackage{graphicx}
\usepackage{url}
\usepackage[hidelinks]{hyperref}
\usepackage[utf8]{inputenc}
\usepackage[small]{caption}
\usepackage{amsmath}
\usepackage{algorithm}
\usepackage{algorithmic}
\usepackage{enumitem}
\usepackage{xcolor}
\usepackage{amssymb}
\usepackage{multirow}
\newcommand\blfootnote[1]{%
  \begingroup
  \renewcommand\thefootnote{}\footnote{#1}%
  \addtocounter{footnote}{-1}%
  \endgroup
}

\newsavebox\CBox
\def\textBF#1{\sbox\CBox{#1}\resizebox{\wd\CBox}{\ht\CBox}{\textbf{#1}}}

\begin{document}
%
% \title{Improving the Attacker by Mining the Propensity of Graph Structure Attacks}
% \title{From Revealing the Mechanisms of Graph Structure Attackers to Improving Methods}
\title{What Does the Gradient Tell When Attacking the Graph Structure}
%
% \titlerunning{From Revealing the Mechanisms of Graph Structure Attackers}
% If the paper title is too long for the running head, you can set
% an abbreviated paper title here
%
% \author{Anonymous Authors}
% \institute{Affiliation}

\author{Zihan Liu\inst{1,2} \and
Ge Wang\inst{1,2} \and
Yun Luo\inst{1,2}\and Stan Z. Li\inst{2}$^\dagger$}
\institute{Zhejiang University, Hangzhou, China \and
AI Lab, Research Center for Inductries of the Future, Westlake University}
\authorrunning{Liu et al.}
% % First names are abbreviated in the running head.
% % If there are more than two authors, 'et al.' is used.
% %
% \institute{Princeton University, Princeton NJ 08544, USA \and
% Springer Heidelberg, Tiergartenstr. 17, 69121 Heidelberg, Germany
% \email{lncs@springer.com}\\
% \url{http://www.springer.com/gp/computer-science/lncs} \and
% ABC Institute, Rupert-Karls-University Heidelberg, Heidelberg, Germany\\
% \email{\{abc,lncs\}@uni-heidelberg.de}}
%
\maketitle              % typeset the header of the contribution

\input{0_abstract}
\input{1_introduction}

\input{3_preliminary}

\input{4_methodology}

\input{5_experiments}

\input{2_related_work}

\input{6_conclusion}

% ---- Bibliography ----

\bibliographystyle{splncs04}
\bibliography{ecml23}
\end{document}

%% file: 0_abstract.tex
\begin{abstract}

Recent research has revealed that Graph Neural Networks (GNNs) are susceptible to adversarial attacks targeting the graph structure. 
A malicious attacker can manipulate a limited number of edges, given the training labels, to impair the victim model's performance.
Previous empirical studies indicate that gradient-based attackers tend to add edges rather than remove them. 
In this paper, we present a theoretical demonstration revealing that attackers tend to increase inter-class edges due to the message passing mechanism of GNNs, which explains some previous empirical observations. 
By connecting dissimilar nodes, attackers can more effectively corrupt node features, making such attacks more advantageous. 
However, we demonstrate that the inherent smoothness of GNN's message passing tends to blur node dissimilarity in the feature space, leading to the loss of crucial information during the forward process.
To address this issue, we propose a novel surrogate model with multi-level propagation that preserves the node dissimilarity information.
This model parallelizes the propagation of unaggregated raw features and multi-hop aggregated features, while introducing batch normalization to enhance the dissimilarity in node representations and counteract the smoothness resulting from topological aggregation. 
Our experiments show significant improvement with our approach. 
Furthermore, both theoretical and experimental evidence suggest that adding inter-class edges constitutes an easily observable attack pattern.
We propose an innovative attack loss that balances attack effectiveness and imperceptibility, sacrificing some attack effectiveness to attain greater imperceptibility. We also provide experiments to validate the compromise performance achieved through this attack loss.

\keywords{Graph adversarial attack  \and Graph structure attack \and Node-level classification.}

\end{abstract}

%% file: 1_introduction.tex
\section{Introduction} \label{sec_intro}
\blfootnote{$^\dag$Corresponding author.}
Graph Neural Networks (GNNs) \cite{kipf2016semi,velickovic2017graph}, a family of deep learning frameworks, are broadly used due to their advantages in processing structure-based data \cite{zhou2020graph}.
GNNs have been widely used in real-world applications, such as social networks \cite{guo2020deep}, traffic networks \cite{bui2021spatial}, and recommendation systems \cite{wu2020graph}, leading to concerns about GNNs' reliability.
In recent years, researchers have pointed out that inconspicuous perturbations in the graph structure can mislead the prediction of GNNs \cite{dai2018adversarial}.
As one of the mainstream approaches to test model robustness, the study of graph adversarial attacks has received more and more attention.

Graph Structure Attack (GSA) aims to misguide the model predictions by perturbing a limited number of edges.
In the scenario of gray-box attacks, the attacker constructs a surrogate model to carry out the attack and afterward transfers the attack to unknown victim models \cite{sun2018adversarial}.
One of the most popular attack methods on graph structure relies on gradient from a surrogate GNN model \cite{chen2020survey}.
Wu et al. \cite{wu2019adversarial} observe that gradient-based attackers constantly add edges instead of removing them.
Subsequently, other researchers have identified this phenomenon  \cite{jin2020graph,pham2020graph,tang2020transferring}.
They explain that the attacker tends to connect nodes with high feature dissimilarity and uses this as a basis to propose the removal of edges between low similarities as a defensive measure.
Tang et al. \cite{tang2020transferring} extend the interpretation of \cite{wu2019adversarial}. 
They argue that adding edges tends to pollute the sharply transmitted information, yet deleting edges only leads to the loss of some information.
Zhu et al. \cite{zhu2022does} notice that the structure attackers leads to reduced homophily in a targeted attack with a sufficient attack budget.

Although there have been some empirical findings regarding GSA, definitive conclusions could not be drawn.
In order to better reveal the mechanism of GSA, this paper first presents a theoretical analysis of the gradient-based GSAs, which is one of the mainstream methods.
We find that for gradient-based attackers, the gradient is more significant among unconnected nodes from different classes.
Since attackers perform greedy attacks based on gradient saliency, GSA methods, such as Metattack~\cite{zugner2019adversarial}, GraD~\cite{liu2022towards}, consume almost the entire attack budget on adding the inter-class edges (an inter-class edge describes the connection state between two nodes from different classes).
In contrast, removing inter-class edges or adding or removing intra-class edges are attack types that hardly appear.
From our demonstration, attackers' tendency is inseparable from the message passing of GNNs, i.e., aggregation.
Information aggregation in GNNs has been shown to lead to oversmoothing between connected nodes \cite{cai2020note,li2018deeper,zhao2019pairnorm}.
It has the advantage of extracting low-frequency information of connected and similar nodes through intra-class edges \cite{bo2021beyond}.
However, inter-class edges render the aggregation a burden.
The information fusion between nodes from different classes causes the local subgraph to deviate from the distribution manifold, thus causing instability in the training process.
Combined with our demonstration, gradient-based attackers exploit this property to express the effectiveness of the attack performance.
Moreover, as mentioned in \cite{wu2019adversarial}, the nodes connected by the edges from the attacker tend to differ in their features.
This also provides empirical support for the effectiveness of the added interclass edges.
The observation from \cite{zugner2018adversarial} that nodes with a high degree are less vulnerable to attacks can also be explained by the fact that adding a new edge under the aggregation mechanism has less effect on high degree nodes.

Inspiredly, the attacker expects to find nodes with higher dissimilarity through the surrogate model and connect them as a suitable perturbation.
Previous gray-box attack methods simply use GCN as a surrogate model \cite{lin2020exploratory,liu2022gradients,liu2022towards,zugner2019adversarial}.
However, we theoretically elaborate that the aggregation of node features leads to the lose of dissimilarity information between nodes in the forward process.
In order to improve the surrogate model, dissimilarity information needs to be propagated under the condition that performance is preserved.
To achieve this, we propose a novel surrogate model with multi-level propagation.
In a graph convolutional layer, we propagate the unaggregated features and the multi-hop aggregated features in parallel.
The former propagates high-frequency information containing node dissimilarities, while the latter propagates low-frequency information to maintain the discriminative performance.
In addition, we introduce batch normalization for the aggregated features to further extend dissimilarity and combat oversmoothness.
Experimental results are provided to verify the improvement of our proposed surrogate model compared with previous methods.

Based on our validation of the effectiveness of adding inter-class edges, we aim to investigate the influence of attack tendency on the imperceptibility. 
In recent studies, graph homophily has been proposed as a measure of imperceptibility \cite{chen2022understanding}. Specifically, among the four attack types mentioned above, adding inter-class edges is the most efficient way to decrease graph homophily, thereby posing a challenge to the imperceptibility of attacks.
To address this issue, we propose a novel attack loss to mitigate the significant impact of the attack on graph homophily. 
The attack loss is designed based on two components, the homophily ratio and pseudo-labels, to interfere with the attacker's greedy decision-making process by influencing the gradient saliency perceived by the attacker. 
Experiments show that the cost of increasing imperceptibility about homogeneity is the loss of attack performance. Ablation experiments and homogeneity rate change analysis are provided to better demonstrate the trade-off between imperceptibility and attack performance.

The main contribution is summarized as follows:

\begin{itemize}
    \item We theoretically demonstrate that gradient-based GSAs tend to increase inter-class margins, which is caused by the aggregation mechanism of GNNs, providing support for empirical observations in previous works.
    \item We consider that adding edges between highly dissimilar nodes can more effectively corrupt the features. However, we have demonstrated that the propagation process of existing surrogate models blurs the information of node dissimilarity. We propose a surrogate model with multi-level propagation to preserve the dissimilarity information while maintaining performance.
    \item We propose a novel homophily-based attack loss to intervene attacker's decisions, aiming to sacrifice a small amount of attack performance in exchange for improving imperceptibility to homophily.

\end{itemize}

%% file: 3_preliminary.tex
\section{Preliminaries} \label{sec_pre}

\subsection{Notations} \label{sec_pre_not}
For a attribute graph $\mathcal{G}$, it can be represented as $\mathcal{G} = (\mathcal{V},\mathcal{E},X)$, where $\mathcal{V}=\{v_1,v_2,... ,v_n\}$ is the set of $N$ nodes, $\mathcal{E} \subseteq {\mathcal{V}}\times{\mathcal{V}}$ is the edge set.
GNNs are expected to predict the class of unlabeled nodes using the known labels in node-level classification tasks under semi-supervision.
The attribute graph has a node feature matrix $X$ and a subset of the labeled nodes with the label set of $Y$.
Each sample is mapped as a node $v_i$ on the graph with its feature $x_i\in{\mathbb{R}^{d}}$ and label $y_i\in{\mathbb{R}^{k}}$, where $k$ is the number of classes.
A binary adjacent matrix $A \in \{0,1\}^{N\times N}$ is denoted to describe the edges between nodes, where $A_{i,j}=1$ if $(i,j)\subseteq \mathcal{E}$.
The mapping function of a GNN model is noted as $f_{\theta}$, while the prediction of $f_{\theta}$ on each node $v_i$ is represented by a distribution $z_i$.
All the notations with $\mathcal{L}$ denote the loss functions.
Specific to attacks via edge perturbation, the perturbed graph is indicated as $\mathcal{G}^{(t)}=(A^{(t)},X)$, where $t$ represents the $t$-th time the attacker greedily modifies an edge in the graph.

\subsection{Gradient-based Attacker}\label{4.1}

For the edge perturbations on an undirected graph, the attack is limited by a budget $\Delta$, whose mathematical expression is:
\begin{equation} \label{eq_edge_budget}
    \lVert A^{(t)}-A^{(0)}\rVert_0 \leq 2\Delta,
\end{equation}
where $\lVert \cdot \rVert_0$ denotes the $\ell_0$ norm function.
Attacks on the graph structure only change a small number of edges, making the attack imperceptible.

Under the constraint elaborated by Equation \eqref{eq_edge_budget}, Zugner et al. \cite{zugner2019adversarial} first propose a greedy algorithm employed on the gradient-based attack model.
The greedy algorithm makes the attacker adopt only the edge with the most significant gradient values as the perturbation in each call to the gradient saliency.
% Each perturbation is considered a bi-level optimization problem in which the model is retrained on the perturbed graph and generates perturbations for the next stage.
Greedy edge perturbation at each stage is formulated as:
\begin{equation} \label{eq_bilevel_optim}
\begin{aligned}%\mathcal{G}^{(t)}
    & \;\;\; A^{(t)}=\mathop{min}\limits_{\Vert A^{(t)}-A^{(t-1)}\Vert_0= 2} \mathcal{L}_{atk}(f_{\theta^{(t)}}(A^{(t)},X)) \\
    & s.t. \; {\theta}^{(t)}=\mathop{argmin}\limits_{\theta}\, \mathcal{L}_{train}(f_\theta(A^{(t-1)},X),Y),
\end{aligned}
\end{equation}
where $\theta^{(t)}$ denotes the parameter matrices of the surrogate model by minimizing a cross-entropy training loss $\mathcal{L}_{train}$.
The general form of $f_\theta$ is a standard Graph Convolutional Network (GCN):
\begin{equation}
    f_{\theta}(A,X)=softmax(\hat{A}\sigma{(\hat{A}XW^{(0)})}W^{(1)}),
\end{equation}
where $\hat{A}=\tilde{D}^{-\frac{1}{2}}(A+I)\tilde{D}^{-\frac{1}{2}}$ is the normalized adjacent matrix.
At stage $(t)$, the attacker selects one edge to modify (two for adjacency matrix due to symmetry), targeting to minimize the attack loss $\mathcal{L}_{atk}$.
There are formats of $\mathcal{L}_{atk}$, such as cross-entropy \cite{zugner2019adversarial}, minmax \cite{xu2019topology}.
In this paper, $\mathcal{L}_{atk}$ is set as a cross-entropy loss, which is formulated as:
\begin{equation}\label{eq_atk_loss}
    \mathcal{L}_{atk}=-\mathcal{L}_{ce}(f_{\theta^{(t)}}(A^{(t)},X),Y'_{test}),
\end{equation}
where $Y'_{test}$ is the pseudo-label of $f_{\theta^{0}}(A^{(0)},X)$ (i.e., the model trained from the clean graph).
The iterations continue until the perturbed graph satisfies $\lVert A^{(t)}-A\rVert_0 = 2\Delta$.
Minimizing $\mathcal{L}_{atk}$ is not a problem that can be solved by a forward process. 
As an approximation, the gradient, which is back-propagated through the attack loss, reflects the expectation of modification to the variation in $\mathcal{L}_{atk}$.
The calculation of gradient on the adjacent matrix $A^{(t)}$ is following:
\begin{equation}\label{eq_meta_grad}
    \mathcal{A}^{(t)}=\nabla_{A^{(t)}}\mathcal{L}_{atk}(f_{\theta^{(t)}}(\mathcal{G}^{(t)})),
\end{equation}
where $\mathcal{A}^{(t)}$ denotes the gradient at stage $(t)$, and $\nabla$ denotes the gradient operator.

On the graph structure, the edges take the value 0 or 1, which means that the perturbation can only cause the edge to change to a determined state.
When the value of an edge is 0, a positive gradient value means that adding this edge is a sound attack strategy; conversely, a negative gradient value means that removing this edge is a sound attack strategy.
The attacker greedily selects the edge with the most significant absolute gradient value from the sound attack strategy as the perturbation.
The expression for the perturbation at the moment (t+1) based on the gradient at the moment (t) is:
\begin{equation} \label{eq_dec_func}
     A^{(t+1)}-A^{(t)} = \mathop{argmax}\limits_{(i,j)} \; \mathcal{A}^{(t)}_{i,j}\;(1-2A^{(t)}_{i,j})),
\end{equation}
where $(1-2A^{(t)})$ converts the edge state from (0,1) to (1,-1), aiming to weight the gradient saliency. Note that at every iteration, the adjacent matrix is modified simultaneously at the symmetrical positions, i.e., $\lVert A^{(t+1)}-A^{(t)}\rVert_0=2$.
The attack process ends at the $\Delta^{th}$ iteration when $A^{(t)}$ satisfies $\lVert A^{(t)}-A^{(0)}\rVert_0 = 2\Delta$.

%% file: 4_methodology.tex
\section{Methodology} \label{sec_met}

\subsection{Insight into Gradient-based Attackers} \label{4.2}

\begin{figure}
    \centering
    \includegraphics[width = 0.57\hsize]{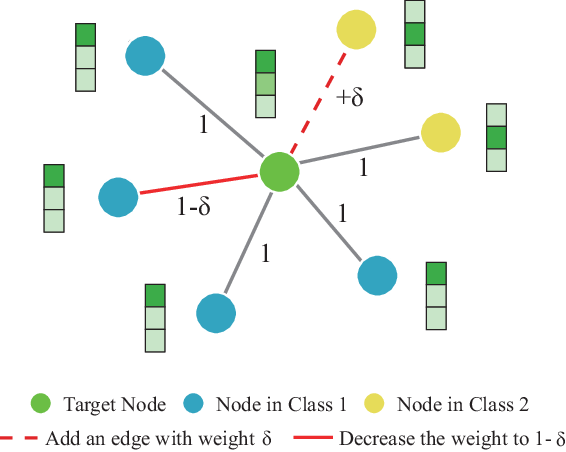}
    \caption{Illustration of the gradient-based attack method. $\delta$ is a minimal amount of perturbation that aims to minimize the prediction confidence on the target node. It can add an edge with weight $\delta$ or decrease the weight of an existing edge to $1-\delta$. }
    \label{fig1}
\end{figure}

We first construct a sub-graph scenario, as shown in Figure \ref{fig1}.
We assume that the label propagation algorithm (LPA) is used to infer the label of the intermediate target node (in green).
The target node is connected to $n_1$ nodes of class 1 and to $n_2$ nodes of class 2.
Prediction confidence in this paper is defined as the probability that a model inferences its prediction of a node sample.
Suppose $n_1>n_2$, according to LPA, the target node has the same label as the dominant category of its neighbors, i.e., class 1, and its prediction confidence is $\frac{n_1}{n_1+n_2}$.
Subsequently, we assume that the edges are non-discrete and try to maximize the probability of the target node being classified as class 1 with a perturbation of size $\delta$.
This small perturbation can be applied to the graph in four approaches: to reduce the weight of an edge connected to a node of class 1 or 2; to connect to a node of class 1 or 2 that is not connected.
Among these, we can easily distinguish that $\delta_{(1)}$: adding an edge connected to a class 2 node with weight $\delta$, and $\delta_{(2)}$: reducing the weight of an edge connected to class 1 to $1-\delta$ are the only two perturbations that can reduce the prediction confidence of the target node.
The influence of two perturbations on the target node's prediction confidence is:
\begin{equation}\label{eq_example}
    p(\delta_{(1)})=\frac{n_1}{n_1+n_2+\delta}\;,\;\;\;\;\;\;\;p(\delta_{(2)})=\frac{n_1-\delta}{n_1-\delta+n_2}\;,
\end{equation}
where $p(\cdot)$ denotes the prediction confidence at class 1.
To compare the attack performance of $\delta_{(1)}$ and $\delta_{(2)}$, we calculate the difference in prediction confidence due to the two perturbations in Equation \eqref{eq_example}. The result of the mathematical expression for $p(\delta_{(1)})-p(\delta_{(2)})$ is simplified as:
\begin{equation}\label{eq_example2}
    p(\delta_{(1)})-p(\delta_{(2)})=\frac{\delta n_2-\delta n_1+\delta^2}{(n_1+n_2)^2-\delta^2}.
\end{equation}
If $\delta \rightarrow0$, then we have $p(\delta_{(1)})< p(\delta_{(2)})$.
Notice that gradient can be demonstrated as the effect of the individual parameters on the objective, reducing prediction confidence in Figure \ref{fig1}.

\begin{lemma} \label{lem1}

Let $P_{\mathcal{L},\delta a}(k)$ be a perturbation function which is defined as $P_{\mathcal{L},\delta a}(k)\equiv\mathcal{L}(a_0,...,a_k+\delta a,...,a_D)-\mathcal{L}(a_0,...,a_k,...,a_D)$, where $a_k$ denotes the $k^{th}$ dimension and $\mathcal{L}$ denotes the attack target function. Then, if $\mathop{argmin}\limits_{k}(P_{\mathcal{L},\delta a}(k)) = d$, the gradient of $\mathcal{L}(a_0,...,a_k,...,a_D)$ w.r.t. $\{a_0,...,a_D\}$ has the minimum at $d^{th}$ dimension $a_d$.
\end{lemma}

\begin{proof}%\renewcommand{\qedsymbol}{}

According to the property of $\mathop{argmin}$, for a positive scaler $\delta{a}$, we have
\begin{equation}
\mathop{argmin}\limits_{k}(\frac{P_{\mathcal{L}, \delta a}(k)}{\delta a}) = d
\nonumber
\end{equation}

The inner part of $\mathop{argmin}$ can be viewed as the  partial derivative, i.e., 
\begin{equation} 
\begin{aligned}
\frac{P_{\mathcal{L}, \delta a}(k)}{\delta a} &= \frac{\mathcal{L}(a_1, ..., a_k+\delta a, ..., a_D) - \mathcal{L}(a_1, ..., a_k, ..., a_D)}{(a_k+\delta a)-(a_k)}\\ &\equiv \frac{\partial \mathcal{L}(a_1, ..., a_k, ..., a_D)}{\partial a_k} \equiv (\nabla_{a} \mathcal{L})_k
\nonumber
\end{aligned}
\end{equation} 
Thus, the gradient of $\mathcal{L}(a_1, ..., a_k, ..., a_D)$ w.r.t $a$ has the minimum at $d^{th}$ dimension of $a$.
\end{proof}

% The proof of Lemma \ref{lem1} is provided in Appendix \ref{append2}.
According to Lemma \ref{lem1} and Equation \ref{eq_example2}, the sub-graph in Figure 1 results in a larger gradient on the unconnected edge to nodes from class 2 (shown by the red dashed line) than a well-connected edge to a node from class 1 (shown by the solid red line). 
Thus, as stated in proposition \ref{propo1}, for an idealized graph, the optimal attack strategy generated by an attacker based on gradient tends to add edges that connect nodes with different labels.

\begin{proposition}  \label{propo1}
For a class of edge perturbation function $P=\{p| p:A {\tiny \rightarrow} A\}$, $A=\{A_{n \times n}|A_{i,j} {\tiny \in} \{0,1\}\}$, the optimal perturbation function follows \\$p^*=\mathop{argmax}\limits_{p{\tiny \in}P} \mathcal{L}(p(A))$. The optimal policy found by the gradient-based attacker obeys:

\begin{equation} \label{score}
p^*=\left\{
\begin{aligned}
&A_{i,j}+1     && ,\;if \; (i,j) \not\in \mathcal{E} \;\; {\rm \&} \;\; y_i\not =y_j \\
&A_{i,j}     && ,\;o.w.
\nonumber
\end{aligned}
\right.
\end{equation}

\end{proposition}

We abstractly represent the disruption of model performance in terms of $\mathop{argmax}\mathcal{L}$.
This expression can also be replaced by $\mathop{argmin}\mathcal{L}_{atk}$ that the perturbations promote the fitting of the attack loss.
Attacking a real-world graph dataset is a non-idealized scenario.
First, real-world graph datasets suffer from label noise. 
Mislabeling leads to significant feature differences between intra-class nodes, which exhibit contamination of neighboring nodes after information aggregation.
Second, the attack on the real-world dataset is restricted to $p:A->A^{(2\Delta)}$, which allows only a few highly significant edges to be perturbed.
As a result, gradient-based edge perturbation manifests itself in real graph attack scenarios with a high probability of adding the inter-class edges.

\subsection{Surrogate Model with Multi-level Propagation}\label{4.3}

The tendency of gradient-based attackers to add inter-class edges has been illustrated in Section \ref{4.2}, which has a visible effect on the gradient of the nodes'  output distribution.
In addition, the difference in attributes between nodes is an essential reference for the attacker.
Edges connecting nodes with significant differences in attributes are more likely to pass contaminative neighbor features to both sides of the edge.
However, the message passing mechanism of the surrogate model, which is generally a Graph Convolutional Network (GCN), blurs the original attribute differences between nodes.
Message passing of a GCN is denoted as:
\begin{equation}
X^{(t+1)}=\tilde{D}^{-\frac{1}{2}} \tilde{A} \tilde{D}^{-\frac{1}{2}}X^{(t)},
\end{equation}
where  $\tilde{A}=A+I$ and $\tilde{D}_i=\sum_{j}A_{ij}$. We can rewrite GCN as:
\begin{equation}
X^{(t+1)}=(I-L_{sym})X^{(t)},
\end{equation} 
where $L_{sym}=\tilde{D}^{-\frac{1}{2}} \tilde{L} \tilde{D}^{-\frac{1}{2}}$, $\tilde{L}=\tilde{D}-\tilde{A}$.

\begin{theorem} \label{theo1}
$\mathcal{G}$ is a non-bipartite and connected graph with $n$ nodes $\mathcal{V}=\{v_1, \dots,v_n\}$, and $X_i^{(0)}$ is the attribute of node $v_i$. Then for $\tau$ large enough, after $\tau$ times message passing, $||X_i^{(\tau)}-X_j^{(\tau)}||_2 \le ||X_i^{(0)}-X_j^{(0)}||_2$.
\end{theorem}

\begin{proof}
Let $(\lambda_1, \dots, \lambda_n)$ and $(e_1, \dots, e_n)$ denote the eigenvalue and eigenvector of matrix $I-L_{sym}$, respectively. According to the property of symmetric Laplacian matrix for connected graph.

$$
-1 < \lambda_1 < \lambda_2 < \cdots < \lambda_n=1 
$$

$$
e_n=\tilde{D}^{-\frac{1}{2}}[1, 1, \dots, 1]^{T}
$$

We can rewrite $X_i^{\tau}-X_j^{\tau}$ as:

$$
X_i^{(\tau)}-X_j^{(\tau)}=[(I-L_{sym})^{\tau}X]_i-[(I-L_{sym})^{\tau}X]_j
$$
$$
\\= [\lambda_1^{\tau}(e_{1}^{(i)}-e_{1}^{(j)}),\dots, \lambda_{n-1}^{\tau}(e_{n-1}^{(i)}-e_{n-1}^{(j)}),\lambda_{n}^{\tau}(e_{n}^{(i)}-e_{n}^{(j)})]\tilde{X}
$$

$e_{k}^{(i)}$ is the $i^{th}$ element of eigenvector $e_{k}$, $\tilde{X}$ is the coordinate matrix of $X$ in the space spanned by eigenvectors $(e_1, \dots, e_n)$.
Thus,

$$
X_{||i-j||}^{(\tau)} := ||X_i^{(\tau)}-X_j^{(\tau)}||_2= \sqrt{\sum_{m=1}^{p}[\sum_{k=1}^{n}\lambda_k^{\tau}(e_{k}^{(i)}-e_{k}^{(j)})\tilde X_{km}]^2}
$$

Because -$1 < \lambda_1 < \lambda_2 < \cdots < \lambda_{n-1} < 1$, thus for a large $\tau$,  $X_{||i-j||}^{(\tau)} < X_{||i-j||}^{(0)}$.

(Further proof is presented in Appendix A.1 in supplementary.)
\end{proof}

Theorem \ref{theo1} demonstrates that the embedding differences between nodes tend to shrink during information aggregation. 
% The proof of Theorem \ref{theo1} can be found in Appendix \ref{append1}.
This shrinking is not of equal magnitude for all node pairs but is specific to the eigenvector.
As the information is forwarded through GCNs, the differences in attributes between nodes are gradually reduced and distorted.
The loss of information in the forward process causes the back-propagated gradients not to contain the lost information.
This prevents gradient-based attackers from exploiting their strengths properly.

To solve this issue, we propose Surrogate Model with Multi-level Propagation (SMMP).
The $l$-th layer of a SMMP model consists of $L$ layers following the form:
\begin{equation} \label{concat}
\hat{H}^{k}_{l-1} = (\hat{A}^k H_{l-1}, BN(\hat{A}^k H_{l-1})),
\end{equation}
\begin{equation} \label{new_surrogate}
H_{l} = \sigma (CONCAT(H_{l-1},\hat{H}^{1}_{l-1},...,\hat{H}^{K}_{l-1} )W).
\end{equation}
In Equation \ref{concat}, $\hat{H}^{k}_{l-1}$ contains $k$-hop aggregated feature $\hat{A}^k H_{l-1}$ with its batch normalization by function $BN(\cdot)$.
The former propagates aggregated features to retain low-frequency information for downstream tasks, while the latter introduces regularization to amplify the difference in node features.
In Equation \ref{new_surrogate}, $\sigma(\cdot)$ denotes the activate function, $CONCAT(\cdot)$ represents concatenation, and $W$ denotes learnable matrices.
$H_{l-1}$ is the feature from the last layer, i.e., high-frequency information, and $\{\hat{H}^{1}_{l-1},... ,\hat{H}^{K}_{l-1}\}$ is the low-frequency information at different aggregation level from 1 to $K$.
Equation \ref{new_surrogate} allows the surrogate to have both node dissimilarity information for attacker and the discriminative information for simulating the victim model.

\subsection{Enhancing Imperceptibility in Perturbed Graph Homophily}\label{4.4}
Section \ref{4.2} has shown that gradient-based attackers tend to add inter-class edges.
A corollary of this phenomenon is an observable decrease in the graph homophily which is a property of the node-level classification task, which reflects the degree of reliability of the graph structure.

The graph homophily $h$ of a given graph $G$ is defined as the proportion of edges in the graph that connects nodes with the same class label (i.e., intra-class edges), defined as:
\begin{equation} \label{eq_homo_ratio}
    h=\frac{|\{(i,j):(i,j)\in \mathcal{E} \wedge y_i=y_j\}|}{|\mathcal{E}|},
% \nonumber
\end{equation}
where $y_i$ is the ground truth or pseudo labels of node $v_i$.
Graphs with strong homophily have a high homophily ratio h→1, while graphs with strong heterophily (i.e., weak homophily) have a small homophily ratio h→0.
Additionally, graphs with a high homophily ratio are termed homophily/assortative graphs, and on the contrary, graphs with a low homology ratio are termed heterophily/disassortative graphs.

Note that the homophily $h$ is related to the label set $Y$ and the edge set $\mathcal{E}$ (i.e., the adjacent matrix $A$). 
We define the homophily ratio of the original graph as $h_Y(A)$ and the homophily ratio of the perturbed graph as $h_Y(A^{\Delta})$.
When the majority of perturbations are adding inter-class edges, the expression for the homophily ratio of the perturbation graph is:
\begin{equation}
h_Y(A^{\Delta})=h_Y(A)\frac{|E|}{|E|+\Delta}
\end{equation} \label{homo_perturbed}
The decrease in the homophily ratio of the perturbed graph relative to that of the original graph is:
\begin{equation} \label{homo_decrease}
h_Y(A)-h_Y(A^{\Delta})=h_Y(A)\frac{\Delta}{|E|+\Delta}
\end{equation}

Research on graph homophily has found that lower homophily tends to imply poorer accuracy because information aggregation introduces significant noise from neighboring nodes \cite{zhu2020beyond}.
In terms of the gradient-based edge perturbation, we ask two questions: (1) Is there a necessary connection between attack performance and the decline in graph homophily? (2) If the decline in homophily is restricted, can the attacker still exhibit attack performance? 
To answer the above questions, we propose a novel attack loss to restrict the homophily variance in the attack process.

The gray-box attack setting limits the agnostic nature of the victim model.
Besides, the ground-truth labels of the test nodes are unknown for both the attacker and the victim model.
Thus, the graph homophily should be calculated from the ground-truth labels of the training nodes and the model-specific pseudo-labels of the test nodes.
The homophily of the original graph is denoted as $h_{\hat{Y}}(A)$ and the homophily of the perturbed graph is denoted as $h_{\hat{Y}}(A^{\Delta})$, where $\hat{Y}=Y_{train} \cup Y'_{test}$, and $Y'_{test}$ denotes the pseudo-labels of test nodes.
The attacker simulates the case on the victim model with the surrogate model.
We expect the decrease in graph homophily of the surrogate model to be restricted by a proportionality parameter $\varepsilon$, denoted as
\begin{equation}\label{eq_homo_cons}
    h_{\hat{Y}}(A)-h_{\hat{Y}}(A^{\Delta})\leq \varepsilon h_{\hat{Y}}(A), \;\; \varepsilon < \frac{\Delta}{|E|+\Delta}.
\end{equation}

The perturbation types that make homophily rise are adding intra-class edges and removing inter-class edges.
In addition to adding inter-class edges, another way to make homophily decrease is to remove intra-class edges. To avoid manufactured decisions on which type of perturbation to adopt in each iteration, we propose a homophily-restricted attack loss to replace the original attack loss in Equation \ref{eq_atk_loss}.
Our proposed attack loss is expressed as:
\begin{equation}
\begin{aligned}
&\mathcal{L}_{atk} = \frac{1}{N}\sum_{v_i}^{N}{P(y_i|f_{\theta^{(t)}}(v_i))},
& \mathcal{L}_{hr} = \Big(\frac{\Vert AH \Vert_0}{\Vert A \Vert_0}-\frac{\Vert A^{(t)} H \Vert_0}{\Vert A^{(t)} \Vert_0}\Big)^2,\\
& \lambda_1 = \Big(1-\frac{h_{\hat{Y}}(A)-h_{\hat{Y}}(A^{(t)})}{\varepsilon h_{\hat{Y}}(A)}\Big)^2,
& \lambda_2 = \Big(\frac{h_{\hat{Y}}(A)-h_{\hat{Y}}(A^{(t)})}{\varepsilon h_{\hat{Y}}(A)}\Big)^2,\\
\end{aligned}
\nonumber
\end{equation}
\begin{equation}\label{eq_attack_loss}
     \mathcal{L}_{atk-hr} = \lambda_1\cdot\mathcal{L}_{atk} + \lambda_2\cdot\mathcal{L}_{hr}.
\end{equation}

In Equation \ref{eq_attack_loss}, $P(y_i|f_{\theta^{(t)}}(v_i))$ represents the confidence level of the surrogate model's prediction on the label class for node $v_i$, and $\mathcal{L}_{ce}$ represents the cross-entropy loss.
The form of $\mathcal{L}_{atk}$ is proposed by \cite{liu2022towards}, which solves a budget allocation problem in attack loss based on cross-entropy.
$\mathcal{L}_{hr}$ is the loss term to restrict graph homophily, where $H \in N\times N$ represents the label consistency between nodes. If node $v_i$ and node $v_j$ have the same label, then $H_{ij}=1$, otherwise $H_{ij}=0$. $H$ is calculated by the pseudo-label $\hat Y$.
The $\Vert AH \Vert_0$ indicates the number of intra-class edges in the graph, and $\Vert A \Vert_0$ indicates the number of edges in the graph.
Thus, term $\frac{\Vert AH \Vert_0 }{ \Vert A \Vert_0}$ is the graph homophily ratio (with self-loop included).
The goal of $\mathcal{L}_{hr}$ is to reduce the homophily difference between the original graph structure $A$ and the perturbed structure $A^{(t)}$ at the $t^{th}$ iteration.
It is worth noting that the attack loss does not aim to optimize the learnable parameters but aims to compute the gradients on the graph structure.
Thus, for both $\mathcal{L}_{hr}$ and $\mathcal{L}_{ce}$, we are merely concerned with their gradients on the graph structure (i.e., $\nabla_{A^{(t)}}\mathcal{L}_{hr}$ and $\nabla_{A^{(t)}}\mathcal{L}_{ce}$).
$\lambda_{1}$ and $\lambda_2$ are a pair of weight parameters.
When $h_{\hat{Y}}(A)\approx h_{\hat{Y}}(A^{(t)})$, we have $\lambda_{1}\rightarrow 1$ and $\lambda_{2}\rightarrow 0$, i.e., the attack at this iteration is not restricted by homophily.
When $h_{\hat{Y}}(A)-h_{\hat{Y}}(A^{(t)})\approx \varepsilon h_{\hat{Y}}(A)$, we have $\lambda_{1}\rightarrow 0$ and $\lambda_{2}\rightarrow 1$, i.e., the attack loss at this iteration focus on restricting graph homophily.
When $0<h_{\hat{Y}}(A)-h_{\hat{Y}}(A^{(t)})< \varepsilon h_{\hat{Y}}(A)$, the attack loss will be a compromise between the two loss terms.
With the synergy of $\lambda_{1}$ and $\lambda_2$, the decline in graph homophily can be restricted by $\varepsilon h_{\hat{Y}}(A)$.
The attacker can achieve edge perturbations that affect homophily decline to various levels by adjusting the parameters $\varepsilon$.
This approach also allows the attacker to trade-off between attack performance and imperceptibility of homophily.

%% file: 5_experiments.tex
\section{Experiments}

The implementation of our proposed surrogate model SMMP and the homophily restriction loss $\mathcal{L}_{atk-hr}$ is open-sourced\footnote{Code is avaliable at \href{https://drive.google.com/file/d/1Jm3-HwU8dUlZ0ImpqZTfGrsEHJUW2uoS/view?usp=share_link}{https://drive.google.com}}. 

% \begin{table}[h]\centering
% \caption{Statistics of datasets.}
% \centering
% \label{stat}
% \begin{tabular}{lcccc}
% \hline
% Datasets & \multicolumn{1}{l}{Vertices} & \multicolumn{1}{l}{Edges} & \multicolumn{1}{l}{Classes} & \multicolumn{1}{l}{Features} \\ \hline
% Citeseer & 3312                         & 4732                      & 6                           & 3703                         \\
% Cora     & 2708                         & 5429                      & 7                           & 1433                         \\
% Cora-ML  & 2995                         & 8416                      & 7                           & 2879                         \\ 
% \hline
% \end{tabular}
% \end{table}
\subsection{Experimental Settings}
To conduct a fair comparison with baselines, we adopt general settings in previous works.
The test scenario in this paper is a gray-box poisoning attack to reduce the classification accuracy of the test nodes in the graph dataset.
The datasets involved in this paper include Cora \cite{mccallum2000automating}, Cora-ML \cite{mccallum2000automating}, Citeseer \cite{sen2008collective}.
% Table \ref{stat} provides the number of nodes, edges, and classes for each dataset.
10\% of the nodes in the assortative datasets Cora, Cora-ML, and Citeseer are labeled, while 90\% of nodes are unlabeled.
The GSA baselines include Random and DICE \cite{waniek2018hiding} based on random sampling as well as EpoAtk \cite{lin2020exploratory}, Meta-Self \cite{zugner2019adversarial}, AtkSE \cite{liu2022gradients}, and GraD \cite{liu2022towards} based on gradient.
% Baselines are detailed as follows.
% \textit{DICE}: A method that removes the intra-class edges and adds the inter-class edges randomly. It uses the pseudo-label from the surrogate model for each unlabeled node.
% \textit{EpoAtk}: In this method, an exploration strategy similar to the genetic algorithm is proposed to avoid the possible error from the gradient information.
% \textit{Meta-Self \& Meta-Train}: Baseline gradient-based attackers attack using the graph structure's gradient information. The two attack models have different attack loss functions.
The attack budget is set as 5\% of the number of edges in the clean graph.
Experimental results show the mean and variance of the 10 times test results.

\subsection{Attack Performance Improved by SMMP}
In this section, we show the attack performance of our attack model with the baseline attack model on general GNNs.
We select GCN \cite{kipf2016semi}, GraphSAGE (G-SAGE) \cite{hamilton2017inductive} and GAT \cite{velickovic2017graph} as victim models.
% 在这一部分 SMMP的优化是基于攻击模型GraD，即采用Equation \ref{eq_attack_loss}中的$\mathcal{L}_{atk}$作为攻击损失。
This experiment aims to verify the effectiveness of the multi-level propagation in the surrogate model, thus confirming that the blurring of node dissimilarity information due to oversmoothing affects the attacker's judgment.
Table \ref{table1} shows the results of SMMP compared with baselines.

\begin{table}[t]
\scriptsize
\centering
\tabcolsep=0.066cm
\caption{Attack performance of our proposed SMMP and baseline methods on victim models GCN, GAT and GraphSAGE (G-SAGE). Experimental results are presented as classification accuracy (\%). The best result in each experimental group is bolded.}
\label{table1}
\begin{tabular}{lccccccccc}
\hline
\multicolumn{1}{l|}{}          & \multicolumn{3}{c|}{Cora}                                                       & \multicolumn{3}{c|}{Cora-ML}                                                    & \multicolumn{3}{c}{Citeseer}                                                   \\ \hline
\multicolumn{1}{l|}{Victim}    
& \multicolumn{1}{c}{GCN} & \multicolumn{1}{c}{GAT} & \multicolumn{1}{c|}{G-SAGE} & \multicolumn{1}{c}{GCN} & \multicolumn{1}{c}{GAT} & \multicolumn{1}{c|}{G-SAGE} & \multicolumn{1}{c}{GCN} & \multicolumn{1}{c}{GAT} & \multicolumn{1}{c}{G-SAGE} \\ \hline
\multicolumn{1}{l|}{Clean}     &81.7$\pm$0.3  &81.4$\pm$0.6  & \multicolumn{1}{l|}{80.8$\pm$0.4}  
& 84.0$\pm$0.4      & 82.3$\pm$0.3   &\multicolumn{1}{l|}{81.9$\pm$0.6 } &69.9$\pm$0.4     & 68.2$\pm$0.6    & 69.8$\pm$0.5                            \\
\multicolumn{1}{l|}{Random}    &81.2$\pm$0.3      &81.7$\pm$0.3   & \multicolumn{1}{l|}{81.2$\pm$0.3}       
& 82.8$\pm$0.4   & 83.0$\pm$0.5   & \multicolumn{1}{l|}{82.8$\pm$0.4}       &67.8$\pm$0.4        &69.0$\pm$0.3        & 67.8$\pm$0.4      \\
\multicolumn{1}{l|}{DICE}      &80.9$\pm$0.5      & 79.6$\pm$0.8      & \multicolumn{1}{l|}{80.0$\pm$0.4}       
&82.5$\pm$0.4       &80.9$\pm$0.5              & \multicolumn{1}{l|}{80.7$\pm$0.9}    
&69.0$\pm$0.4        &66.5$\pm$0.7        & 69.1$\pm$0.8          \\
\multicolumn{1}{l|}{EpoAtk}    &77.0$\pm$0.6      & 79.4$\pm$0.7     & \multicolumn{1}{l|}{78.9$\pm$0.8}       
&81.3$\pm$0.4     &79.4$\pm$0.6    & \multicolumn{1}{l|}{80.2$\pm$0.7}       & 66.3$\pm$0.4  & 66.9$\pm$0.6   & 68.8$\pm$0.5                  \\
\multicolumn{1}{l|}{Meta-Self} & 75.8$\pm$0.4   &77.5$\pm$0.4     & \multicolumn{1}{l|}{74.9$\pm$0.8}       & 76.2$\pm$0.3                        &74.2$\pm$0.5                         & \multicolumn{1}{l|}{76.3$\pm$1.6}       & 60.4$\pm$0.4     &60.3$\pm$0.8    &60.8$\pm$0.6                            \\
\multicolumn{1}{l|}{AtkSE}     &73.7$\pm$0.4        & 76.2$\pm$0.8          & \multicolumn{1}{l|}{73.3$\pm$0.6}       & 74.0$\pm$1.0     & 73.5$\pm$0.9     & \multicolumn{1}{l|}{75.4$\pm$1.0 }       & 59.5$\pm$0.5                   & 60.9$\pm$0.6   & 60.7$\pm$0.4                            \\
\multicolumn{1}{l|}{GraD}      &71.4$\pm$0.5     & 74.1$\pm$0.4      & \multicolumn{1}{l|}{70.6$\pm$1.1}       & 75.0$\pm$0.3       &    72.9$\pm$0.6    & \multicolumn{1}{l|}{74.6$\pm$1.0}       & 55.8$\pm$0.9    & 60.1$\pm$0.8       & 58.8$\pm$2.2       \\
\multicolumn{1}{l|}{SMMP}      & \textBF{68.7$\pm$0.7}      & \textBF{73.4$\pm$0.6}      & \multicolumn{1}{l|}{\textBF{67.0$\pm$1.0}}       & \textBF{68.2$\pm$0.7}        & \textBF{70.3$\pm$0.3}             & \multicolumn{1}{l|}{\textBF{68.1$\pm$1.0}}       &\textBF{55.7$\pm$0.4} 
&\textBF{56.9$\pm$1.0}        &\textBF{56.5$\pm$1.1}       \\ \hline           
\multicolumn{1}{l|}{Gain}      & \textBF{+2.7}      & \textBF{+0.7}      & \multicolumn{1}{l|}{\;\;\;\textBF{+3.6}}       & \textBF{+5.8}        & \textBF{+2.6}             & \multicolumn{1}{l|}{\;\;\;\textBF{+6.5}}       & \textBF{+0.1} 
&\textBF{+3.2}      &\textBF{+2.3}   \\\hline
\end{tabular}
\end{table}

Among the experiments, our proposed attack model SMMP outperforms baselines.
In most of the experimental groups, the attack performance of GraD is the best in the baseline.
For Cora dataset, our method performs 2.7\%, 0.7\% and 3.6\% better than the second-place method GraD on the victim model GCN, GAT and GraphSAGE.
Our method shows a significant improvement over previous methods on the Cora-ML dataset.
Ours outperforms the second-place method by 5.8\%, 2.6\%, and 6.5\% on GCN, GAT and GraphSAGE.
For Citeseer dataset, our method is 3.0\% and 2.1\% above the second-place method on GAT and GraphSAGE.
On GCN, our model slightly outperforms GraD by 0.1\%.
The experimental results show that our proposed SMMP becomes the latest state-of-the-art.
This result also provides support for our demonstration on the propagation of node dissimilarity and the effectiveness of our proposed multi-level propagation.

\subsection{Imperceptibility from Attack Loss $\mathcal{L}_{atk-hr}$}
We implement the attack loss with homophily restriction, $\mathcal{L}_{atk-hr}$, to our proposed SMMP to evaluate the enhancement of imperceptibility and the cost behind it.
Experiments intend to reveal the effect of $\mathcal{L}_{atk-hr}$ to maintain homophily of perturbed graphs as well as the loss on attack performance.
Experiments also aim at evaluating the compromise between attack performance and imperceptibility that the attacker is able to regulate through $\mathcal{L}_{atk-hr}$.
The hyperparameter $\varepsilon$ is essential to regulate the power of the homophily restriction.
We perform an ablation study with various levels of the constant $\varepsilon$, shown in Table \ref{table2}.

In Table \ref{table2}, the homophily ratio of the clean/perturbed graph is denoted as $h$.
The group 'Original' shows the result of using $\mathcal{L}_{atk}$ as attack loss, which is same with the implementation of Table \ref{table1}.
Note that $\mathcal{L}_{atk-hr}$=$\mathcal{L}_{atk}$ when $\varepsilon$$\rightarrow$+$\infty$.
We incrementally take the value of epsilon to observe its effect asymptotically.
From the results, a higher $\varepsilon$ implies a lower $h$, while a lower $h$ implies a lower classification performance.
The attacker still demonstrates attack performance under the homophily restriction. 
When $\varepsilon$=1\%, take Cora
dataset as an example, the homophily ratio $h$ is decreased by 0.017 than the Clean graph, which is 0.025 higher than the Original group.
In comparison, the decrease in homophily shrinks to 40\% of the original, implying a significant increase in imperceptibility.
In terms of attack performance, the accuracy of the perturbed graph at $\varepsilon$=1\% shrank to 8.8\%, 4.8\%, and 10.1\% on GCN, GAT, and GraphSAGE, compared to that of the clean graph.
Compared to the Original group, the attack performance shrinks to 68\%, 60\% and 73\% on the three datasets, respectively.
Thus, we are able to trade significant homophily imperceptibility gains by $\mathcal{L}_{atk}$ at the expense of a small fraction of attack performance.
With the gradual increase of $\varepsilon$, the attack performance is raised, accompanied by a shrinkage of imperceptibility gain.
The experiments validate the effectiveness of our proposed attack loss in enhancing the imperceptibility of attacks.
In addition, the attack loss allows the attacker to choose a compromise between attack performance and imperceptibility, enabling the attacker to be adapted to a wider range of scenarios.

% Please add the following required packages to your document preamble:
% \usepackage{multirow}
\begin{table}[t]
\centering
\caption{Experimental results of applying $\mathcal{L}_{atk-hr}$ with various hyperparameter $\varepsilon$.
$h$ represents the homophily ratio (without counting self-loops) of the clean or perturbed graph calculated by ground-truth labels.}
\label{table2}
\tabcolsep=0.27cm
\begin{tabular}{llccccccc}
\hline
                          &          & Clean & Original & $\varepsilon$=1\% & $\varepsilon$=2\% & $\varepsilon$=3\% & $\varepsilon$=4\% \\ \hline
\multirow{4}{*}{Cora}     & $h$ & 0.744      & 0.702         & 0.727    & 0.721    & 0.716    & 0.712 \\ 
\cline{2-8}
                          & GCN      & 81.7\%      & 68.7\%         & 72.9\%    & 72.6\%    & 71.7\%    & 70.8\%       \\
                          & GAT      & 81.4\%      & 73.4\%         & 76.6\%    & 75.8\%    & 75.5\%    & 75.4\%      \\
                          & G-SAGE   & 80.8\%      & 67.0\%         & 70.7\%    & 69.6\%    & 69.4\%    & 69.1\%     \\ \hline
\multirow{4}{*}{Citeseer} & $h$ & 0.587      & 0.549         & 0.569    & 0.567    & 0.564    & 0.560      \\
\cline{2-8}
                          & GCN      & 69.9\%      & 55.7\%         & 61.4\%    & 60.7\%    & 58.9\%    & 58.0\%       \\
                          & GAT      & 68.2\%      & 56.9\%         & 60.3\%    & 59.4\%    & 59.3\%    & 59.2\%       \\
                          & G-SAGE   & 69.8\%      & 56.5\%         & 60.6\%    & 59.6\%    & 59.1\%    & 58.9\%      \\ \hline
\end{tabular}
\end{table}

\begin{figure}[t]
    \centering
    \includegraphics[width = \hsize]{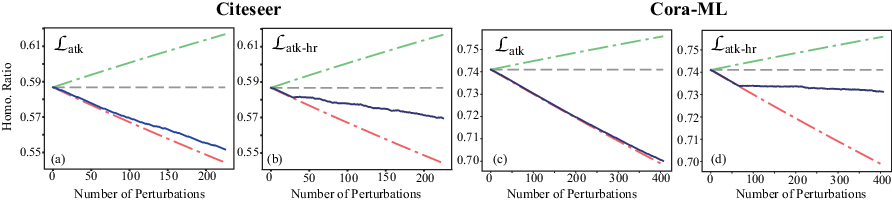}
    \caption{The dynamics of the graph homophily ratio $h$ under $\mathcal{L}_{atk}$ and $\mathcal{L}_{atk-hr}$ on the datasets Citeseer and Cora-ML. Blue line: $h$ of the perturbed graph.; Green and red lines: the upper and lower limits at the time step; Grey line: $h$ of clean graph.}
    \label{fig3}
\end{figure}

To further demonstrate the effect of our proposed attack loss, we visualize the trend of the homophily ratio $h$ during training.
Figure \ref{fig3} shows the dynamics of the homophily ratio $h$ during the attack process of attacker with $\mathcal{L}_{atk}$ and $\mathcal{L}_{atk-hr}$ on the datasets Citeseer and Cora-ML.
The green and red dashed lines represent the upper and lower limits of $h$ at the corresponding attack iteration. 
The gray dashed line indicates the $h$ of the unperturbed graph.
It can be observed from Figure \ref{fig3}(a)(c) that applying $\mathcal{L}_{atk}$, the trend of homophily ratio $h$ is close to the lower limits of $h$, which verifies that the gradient-based attackers tend to add inter-class edges.
Figure \ref{fig3}(b)(d) are collected with $\mathcal{L}_{atk-hr}$ at $\varepsilon$=2\%.
In (b) and (d), the decrease of $h$ becomes flat after triggering a boundary.
This is due to the decrease in homophily triggering $(1-\varepsilon) h$, which results in higher weights for the restriction terms and lower weights for the attack terms in the proposed attack objective.
With this experiment, we empirically demonstrate the attacker's tendency to add inter-class edges, as well as further demonstrate the effectiveness of our proposed attack loss.

%% file: 2_related_work.tex
\section{Related Work}

There are various scenarios of graph adversarial attack depending on the attacker's knowledge, target, and attack scenario \cite{jin2020adversarial,sun2022adversarial,xu2020adversarial}.
The attacker's knowledge can be classified into white-box, gray-box and black-box attacks.
In the gray-box attack that is considered in this paper, the training labels are visible to the attacker.
The attacker's target include targeted attack and untargeted attack.
It can be further classified into the evasion and poisoning attack according to whether the victim model is retrained using the perturbed graph.
Existing attack methods are divided into two main categories: graph modification attack (GMA) and graph node injection (GNI) \cite{jin2020adversarial}.
GMA only allows an attacker to change the properties of existing nodes and edges \cite{lin2020exploratory,ma2020towards}.
Most GMA studies are devoted to attacking GNNs by modifying the graph structure, which is denoted by graph structure attack (GSA). 
Among the GSA methods, gradient-based attacks are one of the dominant methods, first proposed by \cite{zugner2019adversarial}.
Subsequent work focus on the improvements in attack loss and gradient saliency \cite{liu2022gradients,liu2022towards}.

Several graph structure attack researchers have analyzed and interpreted the attack performance and the propensity of the attack model.
Zugner et al. mention in \cite{zugner2018adversarial} that high-degree nodes have better accuracy in both the clean and perturbed graphs.
Wu et al. \cite{wu2019adversarial} observe that gradient-based attackers tend to add edges rather than delete edges.
Xu et al. \cite{xu2019topology} point out that part of the attacker's budget is allocated to nodes that have been misled and propose a Carlile-Wagner-type loss.
Papers \cite{geisler2021robustness,liu2022towards} argue that the attack loss based on cross-entropy is unreasonable.
In addition to attack performance, researches focus on the imperceptibility of the attacks \cite{chen2020survey,sun2018adversarial}.
GSA is generally restricted by the $\ell_0$ norm of the change in the adjacent matrix.
Dai et al. \cite{dai2018adversarial} propose to utilize a benchmark GNN as an evaluation criterion for imperceptibility in white-box attacks.
Ma et al. \cite{ma2019attacking} propose a rewiring approach to attack graph structure in graph classifications.
Chen et al. \cite{chen2022understanding} notice the importance of homophily in GNI and restrict the homophily calculated by the inner product of node features.

%% file: 6_conclusion.tex
\section{Conclusion} 

This paper demonstrates that the mechanism of a gradient-based attacker is to add inter-class edges, which is due to the information aggregation of GNNs.
On this basis, we argue that connecting dissimilar nodes is the expected attack strategy.
However, we theoretically find that the node dissimilarity information is blurred due to aggregation of previous surrogate model.
To address this issue, we improve the surrogate model with multi-level propagation, named SMMP, which preserves the dissimilarity information of high frequency and the discriminative information of low frequency.
Besides, we propose an attack loss with homophily restriction to improve the attack imperceptibility.
The compromise between imperceptibility and attack performance is flexibly controlled by a hyperparameter $\varepsilon$.
We conduct rich experiments, and the results validate the effectiveness of our proposed SMMP and the attack loss with homophily restriction.